\documentclass{article}

\usepackage{microtype}
\usepackage{graphicx}
\usepackage{subfigure}
\usepackage{booktabs} 
\usepackage{hyperref}
\usepackage{amsmath,bm}
\usepackage{amsthm}
\usepackage{amsfonts}       
\usepackage[titletoc,title]{appendix}
\usepackage[american]{babel}

\newcommand{\nobracket}{}
\newcommand{\nocomma}{}
\newcommand{\noplus}{}
\newcommand{\tmmathbf}[1]{\ensuremath{\boldsymbol{#1}}}

\newtheorem{theorem}{Theorem}

\newtheorem{lemma}{Lemma}

\newcommand{\tmop}[1]{\ensuremath{\operatorname{#1}}}
\newcommand{\nosymbol}{}
\DeclareMathOperator{\sign}{sign}

\DeclareMathOperator{\Tr}{Tr}

\usepackage[accepted]{icml2019}

\icmltitlerunning{Spectral Non-convex Optimization for Dimension Reduction with Hilbert-Schmidt Independence Criterion}

\begin{document}

\twocolumn[
\icmltitle{Spectral Non-Convex Optimization for Dimension Reduction with \\
            Hilbert-Schmidt Independence Criterion}



\icmlsetsymbol{equal}{*}

\begin{icmlauthorlist}
\icmlauthor{Chieh Wu}{NEU}
\icmlauthor{Jared Miller}{NEU}
\icmlauthor{Yale Chang}{Philips}
\icmlauthor{Jennifer Dy}{NEU}

\end{icmlauthorlist}

\icmlaffiliation{NEU}{Northeastern University, Boston, MA, 02120}
\icmlaffiliation{Philips}{Philips Research North America, Cambridge, MA, 02141}

\icmlkeywords{Machine Learning, ICML, Dimensionality Reduction, Spectral Clustering, Nonconvex, Optimization}


\icmlcorrespondingauthor{Chieh Wu, Jared Miller}{\{wu.chie,miller.jare\}@husky.neu.edu}
\icmlkeywords{Machine Learning, ICML}

\vskip 0.3in
]



\printAffiliationsAndNotice{}  

\begin{abstract}
The Hilbert Schmidt Independence Criterion (HSIC) is a kernel dependence measure that has applications in various aspects of machine learning. Conveniently, the objectives of different dimensionality reduction applications using HSIC often reduce to the same opti
mization problem. However, the nonconvexity of the objective function arising from non-linear kernels poses a serious challenge to optimization efficiency and limits the potential of HSIC-based formulations. As a result, only linear kernels have been computationally tractable in practice. This paper proposes a spectral-based optimization algorithm that extends beyond the linear kernel. The algorithm identifies a family of suitable kernels and provides the first and second-order local guarantees when a fixed point is reached. Furthermore, we propose a principled initialization strategy, thereby removing the need to repeat the algorithm at random initialization points. Compared to state-of-the-art optimization algorithms, our empirical results on real data show a run-time improvement by as much as a factor of $10^5$ while consistently achieving lower cost and classification/clustering errors. The implementation source code is publicly available on github.\footnote{\href{https://github.com/ANONYMIZED}{https://github.com/endsley}}

\end{abstract}

\section{Introduction}
\label{submission}
The Hilbert-Schmidt Independence Criterion (HSIC) is a kernel-based dependence measure that enables the estimation of dependence between variables without the explicit estimation of their joint distribution. 
As a dependence measure, HSIC has been widely applied to dimensionality reduction in a variety of machine learning paradigms, including supervised dimension reduction~\cite{fukumizu2009kernel,masaeli2010transformation}, unsupervised dimension reduction ~\cite{scholkopf1998nonlinear,niu2011dimensionality}, semi-supervised \cite{gangeh2016semi,chang2017clustering}, and alternative clustering ~\cite{wu2018iterative,niu2010multiple,niu2014iterative}. Although these paradigms are different, when learning a low-dimensional subspace is incorporated into the objective, the objective function for these HSIC-based dimensionality reduction problems can often be formulated as 


    
\begin{equation}
    \underset{W}{\min} - \Tr ( \Gamma K_{X W}) \hspace{0.4cm} \text{s.t.} \hspace{0.2cm} W^T W = I,
     \label{eq:obj_1}
\end{equation}
where $X \in \mathbb{R}^{n \times d}$ is the data set with $n$ samples and $d$ features,
$W \in \mathbb{R}^{d \times q}$ is the projection matrix with orthogonal columns. 
Here,  $\Gamma\in\mathbb{R}^{n\times n}$ is a symmetric matrix of real values and
$K_{XW}\in\mathbb{R}^{n\times n}$ is a kernel matrix with each entry defined as $K_{XW_{ij}}=k(W^Tx_i,W^Tx_j)$ where $k: \mathbb{R}^q \times \mathbb{R}^q \rightarrow \mathbb{R}$ is a kernel function.

Since $q \ll d$, $W$ is a projection matrix that projects $X$ into a lower dimensional subspace. Therefore, the problem is formulated to discover the subspace $W$ while achieving a specific objective. However, the optimization of this formulation is challenging for two reasons.
First, this is due to the fact that the formulation is highly non-convex when the kernel operating on $XW$ is non-linear. Second, it includes an orthogonality constraint, where the solution $W$ must satisfy the condition $W^TW=I$.  As a result, only linear kernels have been tractable computationally. 

In this paper, we propose a generalized Iterative Spectral Method (ISM$^+$) optimization algorithm 
that is fast and easy to implement 
for solving Eq.~(\ref{eq:obj_1}). ISM$^+$ includes both first and second order local guarantees for when a fixed point is reached. 
We identify a family of kernels that are suitable for this algorithm and propose an initialization point $W_0$ based on the second order Taylor approximation of the objective.


\textbf{Related Work.} 

Eq.~(\ref{eq:obj_1}) with its orthogonality constraint is a form of optimization on a manifold: e.g., the constraint can be modeled geometrically as a Stiefel (space of orthogonal coordinate frames) or  Grassmann (space of subspaces, quotient of Stiefel) manifold \cite{james1976topology,nishimori2005learning,edelman1998geometry}.  
Earlier work, \citet{boumal2011rtrmc} propose to recast a similar problem on the Grassmann manifold and then apply first and second-order Riemannian trust-region methods to solve it. \citet{theis2009soft} employs a trust-region method for minimizing the cost function on the Stiefel manifold. \citet{wen2013feasible} later propose to unfold the Stiefel manifold into a flat plane and optimize on the flattened representation. While the manifold approaches perform well under smaller data sizes, they quickly become inefficient when the dimension or sample size increases, which poses a serious challenge to larger modern problems.

Besides manifold approaches, \citet{niu2014iterative} propose Dimension Growth (DG) to perform gradient descent via greedy algorithm a column at a time. By keeping the descent direction of the current column orthogonal to all previously discovered columns, DG ensures the constraint compliance. Although DG is slightly faster with lower dimensional data, it only solves $W$ one column at a time. Therefore, DG slows down quickly as the dimension increases.  

\citet{wu2018iterative} propose an iterative eigendecomposition algorithm, called Iterative Spectral Method (ISM), to solve dimensionality reduction specific for alternative clustering. 
Although ISM is significantly faster than the previous approaches, it lacked generalization due to its kernel specific formulation: i.e., ISM only works on Gaussian kernels.

\textbf{Our Contribution. } 
In this paper, we propose ISM$^+$ to expand the applicability of ISM beyond alternative clustering to solve a general class of subspace dimensionality reduction problems based on HSIC. We generalize ISM beyond the Gaussian kernel by extending the first and second order guarantees to an entire family of kernels. Along with these guarantees, we propose a principled initialization point based on the 2nd order Taylor's approximation, thereby removing the need to repeat the algorithm at random initialization points. We further propose a vectorized reformulation of 
ISM$^+$ that has experimentally shown a runtime improvement against ISM by as much as a factor of $10^4$. Finally, benchmark experiments show that compared to prior alternatives, ISM$^+$ can improve run-time by a maximum factor of $10^5$ while consistently achieving a lower objective cost and classification/clustering errors.

\section{An Overview on HSIC}
Proposed by \citet{gretton2005measuring}, the Hilbert Schmidt Independence Criterion (HSIC) is a statistical dependence measure between two random variables. HSIC is similar to mutual information (MI) because given two random variables $X$ and $Y$, they both measure the distance between the joint distribution $P_{X,Y}$ and the product of their individual distributions $P_X P_Y$. While MI uses KL-divergence to measure this distance, HSIC uses Maximum Mean Discrepancy~\cite{gretton2012kernel}. Therefore, when HSIC is zero, or $P_{X,Y}=P_X P_Y$, it implies independence between $X$ and $Y$. Similar to MI, HSIC score increases as $P_{X,Y}$ and $P_X P_Y$ move away from each other, thereby also increasing their dependence. Although HSIC is similar to MI in its ability to measure dependence,  it is easier to compute as it removes the need to estimate the joint distribution. Due to this advantage, it has been used in many machine learning applications, e.g.,  dimension reduction \cite{niu2011dimensionality}, feature selection \cite{song2007supervised}, and alternative clustering \cite{wu2018iterative}. 


Formally, given a set of $N$ i.i.d. samples $\{(x_1,y_1),...,(x_N,y_N)\}$ drawn from a joint distribution $P_{X,Y}$. Let $X \in \mathbb{R}^{N \times d}$ and $Y  \in \mathbb{R}^{N \times c}$ be the corresponding sample matrices where $d$ and $c$ denote the dimensions of the datasets. We denote by $K_X,K_Y \in \mathbb{R}^{N \times N}$ the kernel matrices with entries $K_{X_{i,j}}=k_X(x_i,x_j)$ and $K_{Y_{i,j}} = k_Y(y_i,y_j)$, where $k_X: \mathbb{R}^d \times \mathbb{R}^d \rightarrow \mathbb{R}$ and $k_Y: \mathbb{R}^c \times \mathbb{R}^c \rightarrow \mathbb{R}$ represent kernel functions. Furthermore, let $H$ be a centering matrix defined as $H=I_n - \frac{1}{n} \textbf{1}_n\textbf{1}_n^T$ where $\textbf{1}_n$ is a column vector of ones. HSIC is computed empirically with
\begin{equation}
    \mathbb{H}(X,Y) = \frac{1}{(n-1)^2} \Tr(K_X H K_Y H).
    \label{eq:emprical_hsic}
\end{equation}
\section{HSIC Dimension Reduction Algorithms}
As mentioned in the introduction, many dimensionality-reduction problems based on HSIC can be reformulated into Eq.~(\ref{eq:obj_1}). In this section, we provide several examples of this relationship. While the examples are not comprehensive, they are designed to maximize the diversity in its applications. For consistency, we maintain the same notations as Eq.~(\ref{eq:obj_1}) throughout all examples, where $K_{XW}$ and $K_Y$ are corresponding kernel matrices computed using $XW$ and $Y$. 

\textbf{Supervised Dimension Reduction. } In supervised dimension reduction \cite{barshan2011supervised,masaeli2010transformation}, both the data $X$ and the label $Y$ are known. We wish to discover a low dimensional subspace $W$ such that $XW$ is maximally dependent to $Y$. As a subspace, we constrain the basis of the columns of $W$ to an orthonormal basis such that $W^TW=I$. This problem can be cast as minimizing the negative HSIC between $XW$ and $Y$ where
\begin{equation}
    \underset{W}{\min}\, - \Tr ( K_{X W} H K_Y H) \qquad \text{s.t.}\quad W^T W = I.
  \label{eq:sdr_1}
\end{equation}
Since $HK_YH$ includes all known variables, they can be considered as a constant $\Gamma = HK_YH$. By rotating the trace terms, we obtain Eq. (\ref{eq:obj_1}). 

\textbf{Unsupervised Dimension Reduction. } 
\citet{niu2011dimensionality} introduced a dimensionality reduction algorithm for spectral clustering based on an HSIC formulation.  In unsupervised dimension reduction, we discover a low dimensional subspace $W$ such that $XW$ is maximally dependent on $Y$. However, unlike the supervised setting, $Y$ in this case is unknown; thus, both $W$ and $Y$ need to be learned simultaneously. We formulate this problem as
\begin{align}
    \underset{W,Y}{\min}\, & -\Tr (K_{X W} H K_Y H ) \\ 
    \text{s.t.} & \qquad W^T W = I, Y^T Y = I,
  \label{eq:udr_1}
\end{align}
where $K_Y = YY^T$.  
This problem is solved by alternating maximization between $Y$ and $W$. When $W$ is fixed, the problem reduces down to spectral clustering and $Y$ can be solved via eigendecomposition as shown in \citet{niu2011dimensionality}. When $Y$ is fixed, the objective becomes the supervised formulation previously discussed.

\textbf{Semi-Supervised Dimension Reduction. } In semi-supervised dimension reduction clustering problems \cite{chang2017clustering}, some form of scores $\hat{Y} \in \mathbb{R}^{n \times r}$ is provided by subject experts for each sample. It is assumed that if two samples are similar, their scores should also be similar. In this case, the objective is to cluster the data given some supervised guidance from the experts. The clustering portion can be accomplished by spectral clustering \cite{von2007tutorial} and HSIC can capture the supervised expert knowledge. By simultaneously maximizing the clustering quality of spectral clustering and the HSIC between the data and the expert scores, this problem is formulated as
\begin{align}
  \underset{W,Y}{\min}\, & -\Tr (Y^T \mathcal{L}_W Y) - \mu \Tr (K_{XW} H K_{\hat{Y}} H),
   \label{eq:ssdr_1}
    \\
    \text{where} & \qquad \mathcal{L}_W = D^{-\frac{1}{2}}K_{XW} D^{-\frac{1}{2}},\\
    \text{s.t} & \qquad W^T W = I, Y^T Y = I,
\end{align}
where $\mu$ is a constant to balance the importance between the first 
and the second terms of the objective, $D\in\mathbb{R}^{n\times n}$ is 
the degree matrix that is a diagonal matrix with its diagonal elements defined as $D_{diag}= K_{XW}\textbf{1}_n$. 
Similar to the unsupervised dimension reduction problem, this objective is solved by alternating optimization of $Y$ and $W$. 
Since the second term does not include $Y$, when $W$ is fixed, the objective reduces down 
to spectral clustering. 
\begin{align}
    \underset{Y}{\min}\, & -\Tr (Y^T D^{-\frac{1}{2}}K_{XW} D^{-\frac{1}{2}} Y)
   \label{eq:ssdr_2}
    \\
    \text{s.t}    & \qquad Y^T Y = I.
\end{align}
By initializing $W$ to an identity matrix, $Y$ is initialized to the solution of spectral clustering. When $Y$ is fixed, $W$ can be solved by isolating $K_{XW}$. If we let $\Psi = HK_{\hat{Y}}H$ and $\Omega=D^{-\frac{1}{2}}Y Y^T D^{-\frac{1}{2}}$, Eq. (\ref{eq:ssdr_1}) can be expressed as 
\begin{align}
    \underset{W}{\min}\, & -\Tr [(\Omega + \mu \Psi) K_{XW} ]
   \label{eq:ssdr_3}
    \\
    \text{s.t} & \qquad W^T W = I.
\end{align}
At this point, it is easy to see that by setting $\Gamma=\Omega+\mu \Psi$, the problem is again equivalent to Eq. (\ref{eq:obj_1}).

\textbf{Alternative Clustering. } In alternative clustering \cite{niu2014iterative}, a set of labels $\hat{Y} \in \mathbb{R}^{n \times k}$ is provided as the original clustering labels. 
The objective of alternative clustering is to discover an alternative set of labels that is high in clustering quality while different from the original label. In a way, this is a form of semi-supervised learning. Instead of having extra information about the clusters we desire, the supervision here indicates what we wish to avoid. Therefore, this problem can be formulated almost identically as a semi-supervised problem with
\begin{align}
    \underset{W,Y}{\min}\, & -\Tr (Y^T \mathcal{L}_W Y) + 
    \mu \Tr (K_{XW} H K_{\hat{Y}} H),
   \label{eq:ac_1}
    \\
    \text{where} & \qquad \mathcal{L}_W = D^{-\frac{1}{2}}K_{XW} D^{-\frac{1}{2}},\\
    \text{s.t}    & \qquad W^T W = I, Y^T Y = I.
\end{align}
Given that the only difference here is a sign change before the second term, this problem can be solved identically as the semi-supervised dimension reduction problem and the sub-problem of maximizing $W$ when $Y$ is fixed can be reduced into Eq.~(\ref{eq:obj_1}).

\section{The Optimization Algorithm}

{\bf Algorithm.}
We propose ISM$^+$, whose pseudo-code is provided in Algorithm~\ref{alg:ism}, to solve Eq.~(\ref{eq:obj_1}).
ISM$^+$ is an iterative spectral method that updates $W_{k}$ for each iteration $k$ based on an eigendecomposition of matrix $\Phi$ as defined on Table~\ref{table:phis}.
We initialize ISM$^+$ by computing $\Phi_0$ based on Table~\ref{table:init_phis} and set the columns of $W_0$ as the $q$ eigenvectors of $\Phi_0$ that is associated with its smallest eigenvalues. At each iteration, we use the previous $W_{k-1}$ to compute the next $\Phi$ based on Table~\ref{table:phis}. Then, the $q$ eigenvectors associated with the smallest eigenvalues of $\Phi$ is again set to $W_k$. In the context of this paper, We will refer to these $q$ eigenvectors as \textit{minimizing eigenvectors} and their eigenvalues as $\Lambda$. This process repeats until $\Lambda$ converges, where convergence is described in the convergence subsection.


{\bf ISM$^+$ Family of Kernels.}
If we let $\beta = a(x_i,x_j)W W^{T} b(x_i, x_j)$, we define the ISM$^+$ family of kernels as kernels that can be expressed as $f(\beta)$.
The $\Phi$ for some common kernels that belong to the ISM$+$ family are provided in Table~\ref{table:phis}. To clarify the notation in Table~\ref{table:phis}, given a matrix $\Psi$, we define $D_\Psi$ as the degree matrix of $\Psi$. While $K_{XW}$ is the kernel computed from $XW$, we denote $K_{XW,p}$ as specifically a polynomial kernel of order $p$. We also denote the symbol $\odot$ as a Hadamard product between 2 matrices. Although these standard kernels are well defined in the literature, we defined each kernel with the projection matrix $W$ in Appendix A for completeness.
    \begin{table}[h]
      \footnotesize
      \centering
      \begin{tabular}{c|l}
        Kernel & Equation\\
        \midrule
            Linear
            	& $\Phi_0=-X^T\Gamma X$ \\
            Squared
            	& $\Phi_0=-X^T[D_\Gamma - \Gamma ]X$ \\  
            Polynomial
            	& $\Phi_0=-X^T \Gamma X$\\
            Gaussian
            	& $\Phi_0=X^T[D_{\Gamma} - \Gamma] X$\\
            \bottomrule       
      \end{tabular}
      \caption{The equation for an initial $\Phi$s depending on the kernel.}
      \label{table:init_phis}
    \end{table}   

    \begin{table}[h]
      \footnotesize
      \centering
      \begin{tabular}{c|l}
        Kernel & Equation\\
        \midrule
            Linear
            	& $\Phi=-X^T\Gamma X$ \\
            Squared
            	& $\Phi=-X^T[D_\Gamma - \Gamma ]X$ \\  
            Polynomial
            	& $\Phi=-X^T\Psi X$ 
             	\hspace{0.2cm}
            	,
            	\hspace{0.2cm}           	
            	$\Psi = \Gamma \odot K_{XW,p-1}$ \\ 
            Gaussian
            	& $\Phi=X^T[D_\Psi - \Psi] X$ 
            	,
            	$\Psi = \Gamma \odot K_{XW}$\\
        \bottomrule       
      \end{tabular}
      \caption{Equation for $\Phi$ depending on the kernel.}
      \label{table:phis}
    \end{table}

\begin{algorithm}[h]
    \footnotesize
   \caption{ISM$^+$ Algorithm}
   \label{alg:ism}
\begin{algorithmic}
   \STATE {\bfseries Input:} data $X$
   \STATE {\bfseries Initialize:} $\Phi_0$ based on Table~\ref{table:init_phis} , $k=0$
   \STATE Set $W_0$ as the minimizing eigenvectors of $\Phi_0$
   \REPEAT
   \STATE Compute $\Phi$ using Table~\ref{table:phis}
   \STATE Set $W_k$ as the minimizing eigenvectors of $\Phi$
   \STATE k = k + 1
   \UNTIL{$\Lambda$ converges (or appropriate termination criteria)}
\end{algorithmic}
\end{algorithm}

{\bf Algorithm Guarantees.}
    The foundation of our algorithm is centered around the concept of the \textit{eigengap}. In this context, given a set of eigenvalues, $\lambda_1 \leq \lambda_2 \leq ...$ from $\Phi$, Algorithm~\ref{alg:ism} picks the eigenvectors corresponding to the least $q$ eigenvalues. Eigengap is  defined as $\mathcal{E} = \lambda_{q+1} - \lambda_q$. Using this definition, we establish the following theorem.
\begin{theorem}\label{thm:stationary}
    Given a full rank $\Phi$ from Table~\ref{table:phis} with an eigengap defined by Eq.~(\ref{eq:conclusion}), a fixed point $W^*$ of Algorithm~\ref{alg:ism} satisfies the 2nd Order Necessary Conditions (Theorem 12.5~\cite{wright1999numerical}) of the objective in Eq. (\ref{eq:obj_1}).
\end{theorem}

\textbf{Proof of Theorem~\ref{thm:stationary}: }
The main body of the proof is organized into two lemmas where the 1st lemma will prove the 1st order condition and the 2nd lemma will prove the 2nd order condition. Instead of proving each individual kernels separately, a general proof will be provided in this section; leaving the specific proof of each kernel to Appendix~\ref{app:deriv_phi}. For convenience, we included the 2nd Order Necessary Condition~\cite{wright1999numerical} in Appendix~\ref{app:2nd_order}.

The proof is initialized by manipulating the different kernels into a common form. If we let $\beta = a(x_i,x_j)W W^{T} b(x_i, x_j)$, then the kernels in this family can be expressed as $f(\beta)$. This common form allows a universal proof that works for all kernels that belongs to the ISM$^+$ family. Depending on the kernel, the definition of $f$, $a(x_i,x_j)$ and $b(x_i,x_j)$ are listed in Table~\ref{table:kernels}. 

 
\begin{table}[h]
    \centering
    \begin{tabular}{l|c|c|c}
    Name & $f(\beta)$ & $a(x_i,x_j)$ & $b(x_i, x_j)$\\ \hline
    Linear & $\beta$ & $x_i$ & $x_j$ \\
    Polynomial & $(\beta + c)^p$ & $x_i$ & $x_j$ \\
    Gaussian &  $e^\frac{-\beta}{2\sigma^2}$ & $x_i-x_j$ & $x_i-x_j$ \\
    Squared & $\beta$ & $x_i-x_j$ & $x_i-x_j$ \\
    \end{tabular}
    \caption{Common components of different Kernels.}
    \label{table:kernels}
\end{table}
 
\begin{lemma} \label{basic_lemma}
    Given $\mathcal{L}$ as the Lagrangian of Eq. (\ref{eq:obj_1}), if $W^{\ast}$ is a fixed point of Algorithm \ref{alg:ism}, and $\Lambda^\ast$ is a diagonal matrix of its corresponding eigenvalues, then
    \begin{align}
        &\nabla_W \mathcal{L} (W^{\ast}, \Lambda^{\ast}) = 0, \label{eq:1st_W}\\
        &\nabla_{\Lambda} \mathcal{L} (W^{\ast}, \Lambda^{\ast}) = 0. \label{eq:1st_lambda}
    \end{align}
\end{lemma}

\begin{proof}
Since $\Tr(\Gamma K_{XW})=\sum_{i,j} \Gamma_{i,j}K_{XW_{i,j}}$, where the subscript indicates the  $i,j$th element of the associated matrix.
If we let $\mathbf{a} = a(x_i, x_j), \mathbf{b} = b(x_i, x_j)$, the Lagrangian of Eq. (\ref{eq:obj_1}) becomes
    \begin{align}
    \begin{split}
        \mathcal{L}(W, \Lambda) = -\sum_{ij}
        \Gamma_{ij} f(\mathbf{a}^{T}WW^{T}\mathbf{b})\\
        - \Tr[\Lambda(W^{T}W-I)]. \label{eq:Lagrangian}
    \end{split}
    \end{align}
The gradient of the Lagrangian with respect to $W$ is
    \begin{align}
    \begin{split}
    \nabla_W \mathcal{L}(W, \Lambda) = 
    -\sum_{ij}{\Gamma_{ij} f'(\mathbf{a}^{T}WW^{T}\mathbf{b})}\\ (\mathbf{b}\mathbf{a}^T + \mathbf{a}\mathbf{b}^T)W - 2 W \Lambda.
    \end{split}
    \label{eq:lagrangian_gradient}
    \end{align}
If we let $A_{i,j}=\mathbf{ba}^T+\mathbf{ab}^T$ then setting $\nabla_W \mathcal{L}(W,\Lambda)$ of Eq. (\ref{eq:lagrangian_gradient}) to 0 yields the relationship
    \begin{align}
    \begin{split}
    0 = \left[-\frac{1}{2}\sum_{ij}{\Gamma_{ij} f'(\mathbf{a}^{T}WW^{T}\mathbf{b})} A_{i,j}\right] W - W \Lambda.
    \end{split}\label{eq:eig_decomp_1}
    \end{align}
Since $f'(\mathbf{a}^{T}WW^{T}\mathbf{b})$ is a scalar value that depends on indices $i,j$, we multiply it by $-\frac{1}{2}\Gamma_{i,j}$ to form a new variable $\Psi_{i,j}$. Then Eq. (\ref{eq:eig_decomp_1}) can be rewritten as
    \begin{align}
    \begin{split}
    \left[\sum_{ij}{\Psi_{ij}} A_{i,j}\right] W = W \Lambda.
    \end{split}\label{eq:eig_decomp_2}
    \end{align}
To match the form shown in Table~\ref{table:phis}, Appendix~\ref{app:xixj} further showed that if $\mathbf{a}$ and $\mathbf{b}$ is equal to $x_i$ and $x_j$, then 
    \begin{align}
    \begin{split}
    \left[\sum_{ij}{\Psi_{ij}} A_{i,j}\right] = 2X^T\Psi X.
    \end{split}\label{eq:xpsix}
    \end{align}
From Appendix~\ref{app:xi-xj}, if $\mathbf{a}$ and $\mathbf{b}$ are equal to $x_i-x_j$, then 
    \begin{align}
    \begin{split}
    \left[\sum_{ij}{\Psi_{ij}} A_{i,j}\right] = 4 X^T[D_\Psi - \Psi]X.
    \end{split}\label{eq:xDpsix}
    \end{align}   
    
If we let $\Phi = \left[\sum_{ij}{\Psi_{ij}} A_{i,j}\right]$, it yields the relationship $\Phi W = W \Lambda$
where the eigenvectors of $\Phi$ satisfies the 1st order condition of $\nabla_W \mathcal{L}(W^\ast,\Lambda^\ast)=0$. The gradient with respect to $\Lambda$ yields the expected constraint    \begin{align}
    \nabla_\Lambda \mathcal{L} = W^TW - I.
    \end{align}
Since the eigenvectors of $\Phi$ is orthonormal, the condition $\nabla_\Lambda \mathcal{L} = 0 = W^TW - I$ is also satisfied. Observing these 2 properties, Lemma~\ref{basic_lemma} confirms that the eigenvectors of $\Phi$ also satisfies the 1st order condition from Eq. (\ref{eq:obj_1}).

\end{proof}

\begin{lemma} \label{eq:2nd_lemma}
    Given a full rank $\Phi$, an eigengap defined by Eq.~(\ref{eq:conclusion}), and $W^*$ as the fixed point of Algorithm~\ref{alg:ism}, then
    \begin{align} 
    \begin{split}
     \tmop{Tr}( Z^T &\nabla_{W W}^2 \mathcal{L}(W^{\ast}, \Lambda^{\ast}) Z) \geq 0 \\ 
     &\tmop{for}   \tmop{all} Z \neq 0 , \tmop{with}  \nabla h (W^{\ast})^T Z = 0.   \label{eq:2nd_W} 
    \end{split}
    \end{align}   
\end{lemma}
\begin{proof}
Due to space constraint, we provide only a summary proof of Lemma~\ref{eq:2nd_lemma} while skipping some tedious details. For a comprehensive proof, please refer to Appendix~\ref{eq:lemma2condition}. To begin, we first note that since $h(W)=W^TW-I$, its directional derivative with respect to $Z$ is 
    \begin{align}
    \nabla h(W)^TZ = \underset{t\rightarrow 0}{\lim}\,
    \frac{\partial}{\partial t} h(W+tZ).
    \end{align}
This can be computed to yield the condition that for all $Z \neq 0$
    \begin{align}
    0 = W^TZ+Z^TW. \label{eq:wzzw}
    \end{align}   
Since $\Phi$ is full rank, its eigenvectors span the full space $\mathcal{R}^d$. Therefore, matrix $Z \in \mathcal{R}^{d \times q}$ can be expressed as 
    \begin{align}
    Z = W B + \bar{W}\bar{B},\label{eq:Z_def}
    \end{align}
where $W$ and $\bar{W}$ represent the eigenvectors chosen and not chosen respectively in Algorithm~\ref{alg:ism}. $B$ and $\bar{B}$ are scrambling matrices, mixing the eigenvectors in $W$ and $\bar{W}$. As $\Phi$ is symmetric and the eigenvectors of a symmetric matrix are orthogonal, $W^{T}\bar{W} = 0$.  Replacing $Z$ from Eq. (\ref{eq:wzzw}) with this relationship, the following derivation shows that $B$ is an antisymmetric matrix.
    \begin{align}
    0 &= W^T Z + Z^T W \\
      &= W^T (W B + \bar{W}\bar{B}) + (W B + \bar{W}\bar{B})^T W \\
    0 &= B + B^T
    \end{align}
The 2nd order directional derivative of the Lagrangian can be computed with
  \begin{align}
     \nabla_{W W}^2 \mathcal{L}(W^{\ast}, \Lambda^{\ast}) Z=
     \underset{t\rightarrow 0}{\lim}\, 
        \frac{\partial}{\partial t} \nabla \mathcal{L}(W+tZ).
  \end{align}

    
    By taking the derivative of $\mathcal{L}(W+tZ)$ with respect to $t$ and setting the limit of $t\rightarrow 0$, we get the 2nd order directional derivative with respect to $Z$ as 
    \begin{align}
    \begin{split}
    &\nabla_{W W}^2 \mathcal{L}(W^{\ast}, \Lambda^{\ast}) Z = -\frac{1}{2}\\
    &\sum_{i, j}{\Gamma_{ij} \left[f''(c_0) c_1 A_{i,j} W + f'(c_0) A_{i,j} Z\right]}
    - Z\Lambda, 
    \end{split}\label{eq:2nd_direct_gradient}
    \end{align}
where $c_0=\mathbf{a}^T W W^T \mathbf{b}$ and $c_1=a^{T} (W Z^{T} + Z W^{T}) b$  are scalar constants with respect to $t$. 
By noticing that $c_0=\beta$, the 2nd term of Eq. (\ref{eq:2nd_direct_gradient}) can be further simplified into
    \begin{align} 
     -\frac{1}{2}\sum_{i,j} \Gamma_{i,j} f'(\beta)A_{i,j} Z = \Phi Z.
    \end{align}   
The term $\Tr(Z^T\nabla^2_{WW}\mathcal{L}(W,\Lambda)Z)$ in Eq. (\ref{eq:2nd_W}) can now be expressed as the sum of 3 equations where
    \begin{align} 
     \tmop{Tr}&( Z^T \nabla_{W W}^2 \mathcal{L}(W^{\ast}, \Lambda^{\ast}) Z) = \mathcal{T}_1 + \mathcal{T}_2 + \mathcal{T}_3, \label{eq:sum_of_3}\\ 
     \mathcal{T}_1 &= \Tr(Z^T\left[-\frac{1}{2}\sum_{i,j} \Gamma_{i,j} f''(c_0)c_1A_{i,j}\right]W),\\
     \mathcal{T}_2 &= \Tr(Z^T\Phi Z),\\
     \mathcal{T}_3 &= -\Tr(Z^T Z\Lambda).
    \end{align}   
Since $\mathcal{T}_1$ cannot be further simplified, the concentration will be on $\mathcal{T}_2$ and $\mathcal{T}_3$. If we let $\bar{\Lambda}$ and $\Lambda$ be the corresponding eigenvalue matrices associated with $\bar{W}$ and $W$, by replacing $Z$ in $\mathcal{T}_2$ from Eq. (\ref{eq:Z_def}), we get
    \begin{align}
    \Tr(Z^{T}\Phi Z) &= \Tr((WB + \bar{W}\bar{B})^T \Phi (WB + \bar{W}\bar{B})) \\
    &= \Tr((WB + \bar{W}\bar{B})^T (W\Lambda B + \bar{W}\bar{\Lambda}\bar{B})) \\
    &= \Tr(B^T \Lambda B ) + \Tr(\bar{B}^T \bar{\Lambda} \bar{B}). \label{eq:T2}
    \end{align}
With $\mathcal{T}_3$, we also replace $Z$ to simplify the expression. To put the expression into a convenient form, we also leverage the fact that $B$ is antisymmetric and therefore $B^{T} = -B$.
    \begin{align}
    \Tr(-Z^{T}Z \Lambda) &= \Tr(-(B^{T}B + \bar{B}^{T} \bar{B})\Lambda) \\
    &= -\Tr{B^{T}B \Lambda} - \Tr(\bar{B}^{T} \bar{B}\Lambda) \\
    &= -\Tr(B \Lambda B^{T}) - \Tr(\bar{B} \Lambda \bar{B}^{T}).
    \label{eq:T3}
    \end{align}
Applying Eq. (\ref{eq:T2}) and (\ref{eq:T3}) to Eq. (\ref{eq:sum_of_3}), it becomes 
    \begin{align} 
    \begin{split}
     \tmop{Tr}( Z^T \nabla_{W W}^2 &\mathcal{L}(W^{\ast}, \Lambda^{\ast}) Z) = \mathcal{T}_1 + \\ 
     &\Tr(B^T \Lambda B ) + \Tr(\bar{B}^T \bar{\Lambda} \bar{B})\\
     &-\Tr(B \Lambda B^{T}) - \Tr(\bar{B} \Lambda \bar{B}^{T}).
     \end{split}
    \end{align}   
Conveniently, some of the terms cancel out and the inequality of Eq. (\ref{eq:2nd_W}) becomes
    \begin{align} 
    \begin{split}
     \Tr(\bar{B}^T \bar{\Lambda} \bar{B}) - \Tr(\bar{B} \Lambda \bar{B}^{T})
     \geq
     \mathcal{T}_1. \\ 
     \end{split}\label{eq:ineqality_3}
    \end{align}   
This expression can be further bounded by
    \begin{align} 
     \Tr(\bar{B}^T \bar{\Lambda} \bar{B}) \geq \underset{i}{\min}\,
    (\bar{\Lambda}_i)\Tr(\bar{B}\bar{B}^T),\\
     \Tr(\bar{B} \Lambda \bar{B}^{T}) \leq \underset{j}{\max}\,
    (\Lambda_j)\Tr(\bar{B}^T\bar{B}).
    \end{align}   
Noting that since $\Tr(\bar{B}^T\bar{B})=\Tr(\bar{B}  \bar{B}^{T})$, it can be treated as a constant value $\alpha$, Eq. (\ref{eq:ineqality_3}) is further simplified into
    \begin{align}
    \left(\underset{i}{\min}\bar{\Lambda}_i - 
    \underset{j}{\max}\,\Lambda_j \right)
    \geq
    \frac{1}{\alpha}\mathcal{T}_1.\label{eq:conclusion}
    \end{align}


    While Lemma~\ref{basic_lemma} propose to use eigenvectors of $\Phi$ as an optimal solution. The contribution of Lemma~\ref{eq:2nd_lemma} informs us \textit{which} eigenvectors should be chosen. Since the 2nd order condition is satisfied when the size of the eigengap is greater than $\frac{1}{\alpha}\mathcal{T}_1$, the eigenvectors should be chosen such that $\underset{j}{\max}\, (\Lambda_j) << \underset{i}{\min}\bar{\Lambda}_i$. This further restricts the potential solutions for Eq.~(\ref{eq:obj_1}) only to the \textit{minimum eigenvectors} of $\Phi$. Therefore, at a fixed point of Algorithm~\ref{alg:ism} given a full rank $\Phi$ and an eigengap that satisfies Eq.~(\ref{eq:conclusion}), the 1st and 2nd order conditions are satisfied.

\end{proof}

\textbf{Subspace Discovery. }
The proof from Lemma~\ref{eq:2nd_lemma} also has significant implications for subspace discovery. Since the exact dimension $q$ of the subspace is unknown 
while discovering the subspace $W \in \mathbb{R}^{d \times q}$, Eq.~(\ref{eq:conclusion}) proposes a guideline to determine the minimally sufficient $q$. However, since the computation of $\mathcal{T}_1$ is challenging, a faster guideline that works well in practice is to set $q$ at the maximum eigengap where 
   \begin{align} 
    q^* = \underset{q}{\arg\max}\,\left(\underset{i}{\min}\, \bar{\Lambda}_i
        - \underset{j}{\max}\, \Lambda_j\right).
    \label{eq:get_q}
    \end{align}
    
\textbf{Computational complexity analysis. } Since both SM and DG required multiple random initializations to discover an optimal solution, we use $i$ to denote the number of associated iterations. We denote the internal iterations required for convergence for each algorithm as $t$. With these two extra notations, the computational complexity of each algorithm for Gaussian and Polynomial kernel is presented in Table~\ref{table:complexity}. Since the Linear and Squared kernels have a closed form solution, their computational complexities are not included. From the experiments, we have observed that while $t$ for SM and DG generally range between 30 to 400, the $t$ value for ISM is generally below 5.
\begin{table}[h]
    \centering
    {\scriptsize
    \begin{tabular}{l|c|c|c|c}
    Kernel & \textbf{ISM$^+$} & ISM & SM & DG\\ \hline
    Gaussian & $O(n^2d t)$ & $O(n^2dq^2 t)$ & $O(n^2dq^2 t i)$ & $O(n^2dq^2 t i)$\\
    Polynomial & $O(n^2d t)$ & $O(n^2dq t)$ & $O(n^2dq t i)$ & $O(n^2dq^2 t i)$\\
    \end{tabular}
    }
    \caption{Computational complexity analysis.}
    \label{table:complexity}
\end{table}


\label{convergence}
\textbf{Convergence Criteria: } Since the objective is to discover a linear subspace, the rotation of the space does not affect the solution. Therefore, instead of constraining the solution on the Stiefel Manifold, the manifold can be relaxed to a Grassmann Manifold. This implies that Algorithm~\ref{alg:ism} can reach convergence as long as the columns space spanned by $W$ are identical. To identify the overlapping span of two spaces, we can append the two matrices into $\mathcal{W} = [W_k W_{k+1}]$ and observe the rank of $\mathcal{W}$. In theory, the rank should equal to $q$, however, a hard threshold on rank often suffers from numerical inaccuracies.

One approach is to study the principal angles (`angles between flats') between the subspaces spanned by $W_k$ and $W_{k+1}$. This is based on the observation that if the maximal principal angle $\theta_{\text{max}} = 0$, then the two subspaces span the same space. The maximal principal angle between subspaces spanned by $W_k$ and $W_{k+1}$ can be found by computing $U \Sigma V^T = W_k^T W_{k+1}$ \cite{knyazev2012principal}. The cosines of the principal angles between $W_k$ and $W_{k+1}$ are the singular values of $\Sigma$, thus $\theta_{\text{max}} = \cos^{-1}(\sigma_{\text{min}})$. Computation of $\theta_{\text{max}}$ requires two matrix multiplications to form $V \Sigma^2 V^T = (W_k^T W_{k+1})^T (W_k^T W_{k+1})$ and then a round of inverse iteration to find $\sigma_{\text{min}}^2$. Although this approach confirms the convergence definitively, in practice, we avoid this extra computation by using the convergence of eigenvalues (of $\Phi$) between iterations as a surrogate. Since eigenvalues are already computed during the algorithm, no additional computations are required. Although tracking eigenvalue of $\Phi$ for convergence is vulnerable to false positive errors, in practice, it works consistently well. Therefore, we recommend to use the eigenvalues as a preliminary check before defaulting to principal angles.

\textbf{Initialization. }
Since the objective Lagrangian is non-convex, a solution can be achieved faster and more accurately if the algorithm is initialized at an intelligent starting point. Ideally, we wish to have a closed-form solution that yields the global optimal without any iterations. However, this is not possible since $\Phi$ is a function of $W$. ISM circumvents this problem by approximating the kernel using Taylor Series up to the 2nd order while expanding around 0. This approximation has the benefit of removing the dependency of $W$ for $\Phi$, therefore, a global minimum can be achieved using the approximated kernel. The ISM algorithm uses the global minimum found from the approximated kernel as the initialization point. Here, we provide a generalized derivation for the entire ISM kernel family. First, we note that the 2nd order Taylor expansion for $f(\beta(W))$ around 0 is
    $f(\beta(W)) \approx f(0) + \frac{1}{2!}\Tr(W^T f''(0) W)$,
where the 1st order expansion around 0 is equal to 0. Therefore, the ISM Lagrangian can be approximated with
    \begin{align} 
    \begin{split} 
    \mathcal{L} = -\sum_{i,j} \Gamma_{i,j} \left[f(0)+\frac{1}{2!}\Tr(W^T f''(0) W)\right]\\
    - \Tr(\Lambda(W^TW-I)),
    \end{split} 
    \end{align}   
where the gradient of the Lagrangian is
    \begin{align} 
    \begin{split} 
    \nabla_W \mathcal{L} = -\sum_{i,j} \Gamma_{i,j} f''(0) W - 2 W \Lambda.
    \end{split} \label{eq:approx_grad}
    \end{align}   
Next, we look at the kernel function $f(\beta(W))$ more closely. The Hessian is computed as
    \begin{align} 
    f'(\beta(W)) = \nabla_\beta f(\beta(W)) \nabla_W \beta(W),\\
    f''(\beta(W=0)) = \nabla_{\beta} f(\beta(0)) \nabla_{W,W} \beta(0).
    \end{align}   
Because $\nabla_{\beta} f(\beta(0))$  is just a constant, we can bundle all constants into this term and refer to it as $\mu$. Since $\nabla_{W,W} \beta(0) = A_{i,j}$, the Hessian is simply $\mu A_{i,j}$ regardless of the kernel. By combining constants setting the gradient of Eq.~(\ref{eq:approx_grad}) to 0, we get the expression
    \begin{align} 
    \begin{split} 
    \left[-\sign(\mu) \sum_{i,j} \Gamma_{i,j} A_{i,j}\right] W = W \Lambda,
    \end{split} \label{eq:initial_eq}
    \end{align}   
where if we let $\Phi=-\sign(\mu) \sum_{i,j} \Gamma_{i,j} f''(0)$, we get a $\Phi$ that is not dependent on $W$. Therefore, a closed-form global minimum of the second-order approximation can be achieved. It should be noted that while the magnitude of $\mu$ can be ignored, the sign of $\mu$ cannot be neglected since it flips the sign of the eigenvalues of $\Psi$. Following Eq. (\ref{eq:initial_eq}), the initial $\Phi_0$ for each kernel is shown in Table~\ref{table:init_phis}. We also provide detailed proofs for each kernel in Appendix~\ref{app:deriv_phi_0}. 

\section{Experiments}
     We compare ISM$^+$ against competing state-of-the-art manifold optimization algorithms: original ISM \textbf{(ISM)} \cite{wu2018iterative}, Dimension Growth \textbf{(DG)} \cite{niu2014iterative}, the Stiefel Manifold approach \textbf{(SM)} \cite{wen2013feasible}, and the Grassmann Manifold \textbf{(GM)} \cite{boumal2011rtrmc} in terms of speed of execution, the magnitude of the objective, and the quality of the result. Since only \citet{manopt} for GM supplied software implementation, we implemented ISM, DG and SM strictly based on the original papers. For the experiments, we have chosen the supervised objective from Eq.~(\ref{eq:sdr_1}) and the unsupervised objective from Eq.~(\ref{eq:udr_1}) to showcase the results.

\textbf{Datasets.}
    The experiment includes 4 real datasets. The Breast Cancer ~\cite{breastcancer}, the Wine ~\cite{Dua:2017} and the Car datasets \cite{Dua:2017} have common structures often encountered in data science. The Face  dataset ~\cite{bay2000uci} includes images of 20 people in various orientations. This dataset was chosen to accentuate ISM's ability to handle high dimensional data. For each experiment, we split the data into training and test. Each algorithm is trained on the training set while we report results on the test set.

\textbf{Evaluation Metric.}
    In the supervised case, we measure quality of results by reporting the test classification error of a Support Vector Machine (SVM) with a Gaussian kernel after the dimension is reduced by the various competing algorithms. 
    In the unsupervised case, we report the Normalized Mutual Information (NMI) \cite{strehl2002cluster} to compare the clustering labels against the ground truth. The NMI is a similarity measure confined within the range of [0,1] with 0 denoting no relationship and 1 as a perfect match. If we let $Z_i$ and $Z_j$ be two clustering assignments, NMI can be calculated with $NMI(Z_i,Z_j) = I(Z_i,Z_j)(H(Z_i)H(Z_j))^{-1/2}$, where $I(Z_i,Z_j)$ is the mutual information between $Z_i$ and $Z_j$, and $H(Z)$ computes the entropy of $Z$. 
    
    Profile of each dataset in Table~\ref{table:supervised_results} and \ref{table:unsupervised_results} use $n$ as the number of samples, $d$ as the original dimension, and $q$ as the final dimension. Times are broken down into days (d), hours (h), and seconds (s). The fastest optimization time, the lowest objective value (cost), and best error/NMI results are bolded. 
    
\textbf{Hyperparameter Settings. }
    All experiments use a Gaussian kernel, where $\sigma$ is the median of the pair-wise Euclidean distance. The dimension of subspace $q$ is selected based on the maximum eigengap. To maintain consistency, all algorithms used the CPU implementation without any extra GPU boost. 

\textbf{Supervised Experiments. }
    We perform supervised dimension reduction via Eq.~(\ref{eq:sdr_1}) and minimize the objective with all competing algorithms. Once the projection matrix is learned, we perform SVM classification on the data with reduced dimension. Among the optimization techniques, we record the runtime, final cost, and the classification error. For comparison, we also classify the data using an SVM without reducing the dimension. Moreover, we include Principal Component Analysis (PCA) with SVM and Linear Discriminant Analysis (LDA) as reference techniques to demonstrate ISM's comparable runtime. The results are reported in Table~\ref{table:supervised_results}. 
    
    From the first five rows of Table~\ref{table:supervised_results}, we note that ISM$^+$ provides a significant speed improvement compared to prior optimization techniques. Due to this improvement, dimension reduction via HSIC becomes comparable in speed to the fastest traditional methods, i.e., PCA, or LDA. This improvement is especially prominent when the original dimension is high. We highlight that for the Face dataset, it took DG 1.78 days, while ISM$^+$ converged within 0.3 seconds. While ISM$^+$ provided up to $10^5$-fold speed improvement, it consistently yields a lower cost. In terms of classification error, while PCA or LDA yields lower error depending on the data, ISM$^+$+SVM always yields the lowest error rate.
    
    \begin{table}[!t]
      \footnotesize
      \centering
      \begin{tabular}{c|c|c|c}
        \toprule
        \multicolumn{4}{c}{Data = Wine, $n=178$, $d=13$, $q=4$}\\
        \toprule
        Algorithms & time & Cost & Error \% \\
        \midrule
           \textbf{ISM$^+$ + SVM}
             	& \textbf{0.006s}
            	& \textbf{-1628.14}
            	& \textbf{0\%} \\
           \textbf{ISM + SVM}
             	& 1.86s
            	& \textbf{-1628.14}
            	& \textbf{0\%} \\           	
           \textbf{DG + SVM}
             	& 7.4s
            	& -1491.2
            	& 1.7\% \\           	
           \textbf{SM + SVM}
             	& 44s
            	& -1627.8
            	& 0.6\% \\           	       
           \textbf{GM + SVM}
             	& 1008s
            	& -1622
            	& 1.7\% \\           	                  	
            LDA
             	& 0.002s
            	& -
            	& \textbf{0\%} \\
            PCA + SVM
             	& 0.001s
            	& -
            	& 1.7\% \\
            SVM
             	& 0.003s
            	& -
            	& \textbf{0\%} \\
        \toprule
        \toprule
        \multicolumn{4}{c}{Data = Cancer, $n=683$, $d=9$, $q=4$}\\
        \midrule
            \textbf{ISM$^+$ + SVM}
             	& \textbf{0.005s}
            	& \textbf{-1645.7}
            	& \textbf{1.5\% } \\
           \textbf{ISM + SVM}
              	& 0.92s
            	& \textbf{-1645.7}
            	& \textbf{1.5\% } \\          
            \textbf{DG + SVM}
             	& 3.7s
            	& -1492
            	& \textbf{1.5\% } \\           	
            \textbf{SM + SVM}
             	& 22.75s
            	& -1644.8
            	& \textbf{1.5\% } \\           	       
          \textbf{GM + SVM}
             	& 1002s
            	& -1641
            	& \textbf{1.5\%} \\           	            
            LDA
             	& 0.001s
            	& -
            	& 3\% \\
            PCA + SVM
             	& 0.001s
            	& -
            	& \textbf{1.5\% } \\
            SVM
             	& 0.001s
            	& -
            	& \textbf{1.5\% } \\
         \toprule
        \toprule
        \multicolumn{4}{c}{Data = Car, $n=1728$, $d=6$, $q=4$}\\
        \midrule
            \textbf{ISM$^+$ + SVM}
             	& \textbf{0.02s}
            	& \textbf{-2304.1}
            	& \textbf{0\%} \\
            \textbf{ISM + SVM}
             	& 95.2s
            	& \textbf{-2304.1}
            	& \textbf{0\%} \\           	
            \textbf{DG + SVM}
             	& 32.5s
            	& -2014.9
            	& 2.3\% \\           	
            \textbf{SM + SVM}
             	& 52.4s
            	& -2268.2
            	& \textbf{0\%} \\           	           
            \textbf{GM + SVM}
             	& 1344s
            	& -2275
            	& 13\% \\           	            
            LDA
             	& 0.03s
            	& -
            	& \textbf{0\%} \\
            PCA + SVM
             	& 0.001s
            	& -
            	& 5\% \\
            SVM
             	& 0.008s
            	& -
            	& \textbf{0\%} \\
        \toprule
        \toprule
        \multicolumn{4}{c}{Data = Face , $n=624$, $d=960$, $q=20$}\\
        \midrule
            \textbf{ISM$^+$ + SVM}
             	& \textbf{0.3s}
            	& \textbf{-4685}
            	& \textbf{0\%} \\
             \textbf{ISM + SVM}
             	& 13320s/3.7h
            	& \textbf{-4685}
            	& \textbf{0\%} \\           	
            \textbf{DG + SVM}
             	& 150774s/1.78d
            	&  -4280
            	& \textbf{0\%} \\           	
            \textbf{SM + SVM}
             	& 49681s/13.8h
            	& -4680
            	& \textbf{0\%} \\           	           	
            \textbf{GM + SVM}
             	& 1140s
            	& -1011
            	& \textbf{0.5\%} \\           	         
            LDA
             	& 0.214s
            	& -
            	& \textbf{0\%}\\
            PCA + SVM
             	& 0.078s
            	& -
            	& \textbf{0\%} \\
            SVM
             	& 0.69s
            	& -
            	& 0.2\% \\           	
        \toprule
        \bottomrule       
      \end{tabular}
      \caption{SVM classification after dimension reduction.}
      \label{table:supervised_results}
    \end{table}

 \textbf{Unsupervised Experiments. }
    We report unsupervised analysis via optimizing Eq.~(\ref{eq:udr_1}) in Table~\ref{table:unsupervised_results}. We solve the objective via alternating minimization by initializing $W$ as the identity matrix. Once $W$ has converged to a fixed point, we project the data onto the lower dimension and perform Spectral Clustering. Besides the 5 optimization techniques, we also report Spectral Clustering (SC) without any dimension reduction as well as Spectral Clustering with dimensions reduced by PCA (PCA + SC). 
   
    Similar to Supervised experiments, ISM$^+$ consistently obtains the lowest cost while requiring significantly less training time. Although a lower objective cost does not guarantee a higher NMI against the ground truth, it is observed that ISM$^+$ generally outperforms other optimization techniques. 
    
    \begin{table}[!t]
      \footnotesize
      \centering
      \begin{tabular}{c|c|c|c}
        \toprule
        \multicolumn{4}{c}{Data = Wine, $n=178$, $d=13$, $q=4$}\\
        \toprule
        Algorithms & time & Cost & NMI\\
        \midrule
           \textbf{ISM$^+$ + SC}
             	& \textbf{0.037s}
            	& \textbf{-39.1}
            	& \textbf{0.88} \\
           \textbf{ISM + SC}
             	& 5.3s
            	& \textbf{-39.1}
            	& \textbf{0.88} \\           	
           \textbf{DG + SC}
             	& 10.63s
            	& -38.07
            	& 0.835 \\           	
           \textbf{SM + SC}
             	& 21.33s
            	& \textbf{-39.1}
            	& \textbf{0.88} \\           	       
           \textbf{GM + SC}
             	& 2022s
            	& \textbf{-39.1}
            	& 0.878\\           	                  	
            PCA + SC
             	& 0.03s
            	& -
            	& 0.835 \\
            SC
             	& 0.003s
            	& -
            	& 0.835 \\
        \toprule
        \toprule
        \multicolumn{4}{c}{Data = Cancer, $n=683$, $d=9$, $q=4$}\\
        \midrule
            \textbf{ISM$^+$ + SC}
             	& \textbf{0.016s}
            	& \textbf{-33.3}
            	& 0.862 \\
            \textbf{ISM + SC}
             	& 1.3s
            	& \textbf{-33.3}
            	& 0.862 \\           	
            \textbf{DG + SC}
             	& 4.26s
            	& -32.86
            	& \textbf{0.899} \\           	
            \textbf{SM + SC}
             	& 3.2s
            	& \textbf{-33.3}
            	& 0.862 \\           	       
            \textbf{GM + SC}
             	& 3.2s
            	& \textbf{-33.3}
            	& 0.862 \\           	                  	
            PCA + SC
             	& 0.001s
            	& -
            	& 0.862 \\
            SC
             	& 0.001s
            	& -
            	& 0.862 \\
         \toprule
        \toprule
        \multicolumn{4}{c}{Data = Car, $n=1728$, $d=6$, $q=4$}\\
        \midrule
            \textbf{ISM$^+$ + SC}
             	& \textbf{0.11s}
            	& \textbf{-79.18}
            	& \textbf{0.35} \\
            \textbf{ISM + SC}
             	& 210.1s
            	& \textbf{-79.18}
            	& \textbf{0.35} \\           	
            \textbf{DG + SC}
             	& 7366s
            	& -76.1
            	& 0.10 \\           	
            \textbf{SM + SC}
             	& 8.4s
            	& -79.1
            	& 0.315 \\           	           	 
            \textbf{GM + SC}
             	& 2826s
            	& -79
            	& 0.15 \\           	           	           	            	
            PCA + SC
             	& 0.07s
            	& -
            	& 0.29 \\
            SC
             	& 0.06s
            	& -
            	& 0.28 \\
        \toprule
        \toprule
        \multicolumn{4}{c}{Data = Face , $n=624$, $d=960$, $q=20$}\\
        \midrule
            \textbf{ISM$^+$ + SC}
             	& \textbf{0.3s}
            	& \textbf{-171}
            	& \textbf{0.95} \\
            \textbf{ISM + SC}
             	& 15821.6s/4.4h
            	& \textbf{-171}
            	& \textbf{0.95} \\           	
            \textbf{DG + SC}
             	& 158196s/1.83d
            	& -169.6
            	& 0.926 \\           	
            \textbf{SM + SC}
             	& 26733 s/ 7.4 h
            	& -170.98
            	&  \textbf{0.95}\\           	           	
            \textbf{GM + SC}
             	& 174859s/2.02d
            	& -37.4
            	& 0.89 \\
            PCA + SC
             	& 0.14s
            	& -
            	& 0.925 \\
            SC
             	& 0.2s
            	& -
            	& \textbf{0.95} \\           	
        \toprule
        \bottomrule       
      \end{tabular}
      \caption{Spectral clustering after dimension reduction.}
      \label{table:unsupervised_results}
    \end{table}   

\section{Conclusion}
We showed that subspace dimensionality reduction based on HSIC for a variety of machine learning paradigms can be re-expressed into a common cost function (Eq.~(\ref{eq:obj_1})).  We propose an iterative spectral algorithm, ISM$^+$, to solve this non-convex optimization problem constrained on a Grassmann Manifold.  We identified a family of kernels that satisfies the first and second-order local guarantees at the fixed point of ISM$^+$. Our experiments demonstrated ISM$^+$'s superior training time with a consistently lower cost while achieving comparable performance against the state-of-art algorithms.

\clearpage
\bibliography{example_paper.bib}
\bibliographystyle{icml2019}

\clearpage
\onecolumn
\begin{appendices}
\section{Kernel Definitions }
Here we provide the definition of each kernel with relation to the projection matrix $W$ in terms of the kernel and as a function of $\beta=\mathbf{a}W W^T \mathbf{b}$.\\
\textbf{Linear Kernel}
\begin{equation}
    k(x_i,x_j) = x_i^TWW^Tx_j, \hspace{1cm} f(\beta) = \beta.
    \label{eq:linear_kernel}
\end{equation}

\textbf{Polynomial Kernel}
\begin{equation}
    k(x_i,x_j) = (x_i^TWW^Tx_j + c)^p, \hspace{1cm} f(\beta) = (\beta + c)^p.
    \label{eq:poly_kernel}
\end{equation}

\textbf{Gaussian Kernel}
\begin{equation}
    k(x_i,x_j) = e^{-\frac{(x_i-x_j)^TW W^T(x_i-x_j)}{2\sigma^2}},
    \hspace{1cm}
    f(\beta) = e^{-\frac{\beta}{2\sigma^2}}.
    \label{eq:gaussian_kernel}
\end{equation}

\textbf{Squared Kernel}
\begin{equation}
    k(x_i,x_j) = (x_i-x_j)^TW W^T(x_i-x_j), \hspace{1cm} f(\beta) = \beta.
    \label{eq:squared_kernel}
\end{equation}

\textbf{Multiquadratic Kernel}
\begin{equation}
    k(x_i,x_j) = \sqrt{(x_i-x_j)^TW W^T(x_i-x_j) + c^2}, \hspace{1cm} f(\beta) = \sqrt{\beta + c^2}.
    \label{eq:squared_kernel}
\end{equation}

\end{appendices}


\begin{appendices}
\section{Theorem 12.5 }\label{app:2nd_order}
\begin{lemma} 
  [Nocedal,Wright, Theorem 12.5~{\cite{wright1999numerical}}] (2nd Order Necessary Conditions)
    Consider the optimization problem:
  $ \min_{W : h (W) = 0} f (W), $
where $f : \mathbb{R}^{d \times q} \to \mathbb{R}$ and $h :
  \mathcal{R}^{d \times q} \to \mathbb{R}^{q \times q}$ are twice continuously
  differentiable. Let   $\mathcal{L}$ be the Lagrangian and $h(W)$ its equality constraint. Then, a local minimum must satisfy the following  conditions:
  \begin{subequations}
    \begin{align} 
    &\nabla_W \mathcal{L} (W^{\ast}, \Lambda^{\ast}) = 0, \label{eq:1st_W}\\
    &\nabla_{\Lambda} \mathcal{L} (W^{\ast}, \Lambda^{\ast}) = 0, \label{eq:1st_lambda}\\
    \begin{split}
     \tmop{Tr}( Z^T &\nabla_{W W}^2 \mathcal{L}(W^{\ast}, \Lambda^{\ast}) Z) \geq 0 \\ 
     &\tmop{for}   \tmop{all} Z \neq 0 , \tmop{with}  \nabla h (W^{\ast})^T Z = 0.   \label{eq:2nd_W_in_append} 
    \end{split}
    \end{align}   
  \end{subequations}
\end{lemma}
\end{appendices}

\begin{appendices}
\section{Derivation for $\sum_{i,j} \Psi_{i,j} A_{i,j}$ if $A_{i,j}=x_ix_j^T+x_jx_i^T$}
Since $\Psi$ is a symmetric matrix and $A_{i, j} = ( x_i x_j^T + x_j x_i^T)$,
we can rewrite the expression into
\[ \sum_{i, j} \Psi_{i, j} A_{i, j} = 2 \sum_{i, j}^n \Psi_{i, j} x_i x_j^T .
\]
If we expand the summation for $i = 1$, we get
\[ \begin{array}{lll}
     {}[ \Psi_{1, 1} x_1 x_1^T + \ldots + \Psi_{1, n} x_1 x_n^T] & = & x_1 [
     \Psi_{1, 1} x_1^T + \ldots + \Psi_{1, n} x_n^T]\\
     & = & x_1 \left[ \left[ \begin{array}{lll}
       x_1 & \ldots & x_n
     \end{array} \right] \left[ \begin{array}{l}
       \Psi_{1, 1}\\
       .\\
       \Psi_{1, n}
     \end{array} \right] \right]^T\\
     & = & x_1 \left[ \left[ \begin{array}{lll}
       \Psi_{1, 1} & \ldots & \Psi_{1, n}
     \end{array} \right] \left[ \begin{array}{l}
       x_1^T\\
       .\\
       x_n^T
     \end{array} \right] \right] .
   \end{array} \]
Now if we sum up all $i$, we get
\[ \begin{array}{lll}
     \Psi_{i, j} x_i x_j^T & = & x_1 \left[ \left[ \begin{array}{lll}
       \Psi_{1, 1} & \ldots & \Psi_{1, n}
     \end{array} \right] \left[ \begin{array}{l}
       x_1^T\\
       .\\
       x_n^T
     \end{array} \right] \right] + \ldots + x_n \left[ \left[
     \begin{array}{lll}
       \Psi_{n, 1} & \ldots & \Psi_{n, n}
     \end{array} \right] \left[ \begin{array}{l}
       x_1^T\\
       .\\
       x_n^T
     \end{array} \right] \right],\\
     & = & \left[ x_1 \left[ \begin{array}{lll}
       \Psi_{1, 1} & \ldots & \Psi_{1, n}
     \end{array} \right] + \ldots + x_n \left[ \begin{array}{lll}
       \Psi_{n, 1} & \ldots & \Psi_{n, n}
     \end{array} \right] \right] \left[ \begin{array}{l}
       x_1^T\\
       .\\
       x_n^T
     \end{array} \right],\\
     & = & \left[ \left[ \begin{array}{lll}
       x_1 & \ldots & x_n
     \end{array} \right] \left[ \begin{array}{l}
       \Psi_{1, 1}\\
       .\\
       \Psi_{n, 1}
     \end{array} \right] + \ldots + \left[ \begin{array}{lll}
       x_1 & \ldots & x_n
     \end{array} \right] \left[ \begin{array}{l}
       \Psi_{1, n}\\
       .\\
       \Psi_{n, n}
     \end{array} \right] \right] \left[ \begin{array}{l}
       x_1^T\\
       .\\
       x_n^T
     \end{array} \right],\\
     & = & \left[ \begin{array}{lll}
       x_1 & \ldots & x_n
     \end{array} \right] \left[ \begin{array}{lll}
       \left[ \begin{array}{l}
         \Psi_{1, 1}\\
         .\\
         \Psi_{n, 1}
       \end{array} \right] & \ldots & \left[ \begin{array}{l}
         \Psi_{1, n}\\
         .\\
         \Psi_{n, n}
       \end{array} \right]
     \end{array} \right] \left[ \begin{array}{l}
       x_1^T\\
       .\\
       x_n^T
     \end{array} \right] .
   \end{array} \]
Given that $X = \left[ \begin{array}{lll}
  x_1 & \ldots & x_n
\end{array} \right]^T$, the final expression becomes.
\[ 2 \sum_{i, j}^n \Psi_{i, j} x_i x_j^T = 2 X^T \Psi X. \]
\label{app:xixj}
\end{appendices}

\begin{appendices}
\section{Derivation for $\sum_{i,j} \Psi_{i,j} A_{i,j}$ if $A_{i,j}=(x_i-x_j)(x_i-x_j)^T+(x_i-x_j)(x_i-x_j)^T$}
Since $\Psi$ is a symmetric matrix, and $A_{i, j} = ( x_i - x_j) ( x_i -
x_j)^T + ( x_i - x_j) ( x_i - x_j)^T = 2 ( x_i - x_j) ( x_i - x_j)^T$, we can
rewrite the expression into
\[ \begin{array}{lll}
     \sum_{i, j} \Psi_{i, j} A_{i, j}  & = & 2 \sum_{i, j} \Psi_{i, j} ( x_i -
     x_j) ( x_i - x_j)^T\\
     & = & 2 \sum_{i, j} \Psi_{i, j} ( x_i x_i^T - x_j x_i^T - x_i x_j^T +
     x_j x_j^T)\\
     & = & 4 \sum_{i, j} \Psi_{i, j} ( x_i x_i^T - x_j x_i^T)\\
     & = & \left[ 4 \sum_{i, j} \Psi_{i, j} ( x_i x_i^T) \right] - \left[ 4
     \sum_{i, j} \Psi_{i, j} ( x_j x_i^T) \right] .
   \end{array} \]
If we expand the 1st term where $i = 1$, we get
\[ \sum_{i = 1, j}^n \Psi_{1, j} ( x_1 x_1^T) = \Psi_{1, 1} ( x_1 x_1^T) +
   \ldots + \Psi_{1, n} ( x_1 x_1^T) = \left[ \sum_{i = 1, j}^n \Psi_{1, j}
   \right] x_1 x_1^T . \]
From here, we notice that $\left[ \sum_{i = 1, j}^n \Psi_{1, j} \right]$ is
the degree $d_{i = 1}$ of $\Psi_{i = 1}$. Therefore, if we sum up all $i$
values we get
\[ \sum_{i, j} \Psi_{i, j} ( x_i x_i^T) = d_1 x_1 x_1^T + \ldots + d_n x_n
   x_n^T \nosymbol . \]
If we let $D_{\Psi}$ be the degree matrix of $\Psi$, then this expression
becomes
\[ 4 \sum_{i, j} \Psi_{i, j} ( x_i x_i^T) = 4 X^T D_{\Psi} X. \]
Since Appendix~\ref{app:xixj} has already proven the 2nd term, together we get
\[ 4 \sum_{i, j} \Psi_{i, j} ( x_i x_i^T) - 4 \sum_{i, j} \Psi_{i, j} ( x_j
   x_i^T) = 4 X^T D_{\Psi} X - 4 X^T \Psi X = 4 X^T [ D_{\Psi} - \Psi] X. \]
\label{app:xi-xj}
\end{appendices}

\begin{appendices}
\section{Derivation for each $\Phi_0$}
\label{app:deriv_phi_0}
Using Eq.~(\ref{eq:initial_eq}), we know that
    \begin{equation}
        \Phi_0 = -sign(\mu) \sum_{i,j} \Gamma_{i,j} A_{i,j}.
    \end{equation}
If $\mathbf{a}$ and $\mathbf{b}$ are both defined as $x_i-x_j$, then 
    \begin{equation}
        \Phi_0 = -sign(4 \mu) X^T(D_\Gamma - \Gamma) X.
        \label{eq:phi_0_form_1}
    \end{equation}
However, if $\mathbf{a}$ and $\mathbf{b}$ are defined as $(x_i,x_j)$, then 
    \begin{equation}
        \Phi_0 = -sign(2 \mu) X^T\Gamma X.
        \label{eq:phi_0_form_2}
    \end{equation}
Therefore, to compute $\Phi_0$, the key is to first determine the ($\mathbf{a}$ , $\mathbf{b}$) based on the kernel and then find $\mu$ to determine the sign.

\textbf{$\mathbf{\Phi_0}$ for the Linear Kernel: }
With a Linear Kernel, $(\mathbf{a},\mathbf{b})$ uses $(x_i,x_j)$, therefore Eq.~(\ref{eq:phi_0_form_2}) is use. Since $f(\beta) = \beta$, the gradient with respect to $\beta$ is 
    \begin{equation}
        -sign(2 \nabla_\beta f(\beta)) = -sign(2) = -1.
    \end{equation}
Therefore, 
    \begin{equation}
        \Phi_0 = -X^T\Gamma X.
    \end{equation}

\textbf{$\mathbf{\Phi_0}$ for the Polynomial Kernel: }
With a Polynomial Kernel, $(\mathbf{a},\mathbf{b})$ uses $(x_i,x_j)$, therefore Eq.~(\ref{eq:phi_0_form_2}) is use. Since $f(\beta) = (\beta + c)^p$, the gradient with respect to $\beta$ is 
    \begin{equation}
        -sign(2 \nabla_\beta f(\beta)) = -sign(2p(\beta+c)^{p-1}) = -1.
    \end{equation}
Therefore, 
    \begin{equation}
        \Phi_0 = -X^T\Gamma X.
    \end{equation}

\textbf{$\mathbf{\Phi_0}$ for the Gaussian Kernel: }
With a Gaussian Kernel, $(\mathbf{a},\mathbf{b})$ uses $x_i-x_j$, therefore Eq.~(\ref{eq:phi_0_form_1}) is use. Since $f(\beta) = e^{-\frac{\beta}{2\sigma^2}}$, the gradient with respect to $\beta$ is 
    \begin{equation}
        -sign(4 \nabla_\beta f(\beta)) = -sign(-\frac{4}{2\sigma^2}e^{-\frac{\beta}{2\sigma^2}}) = 1.
    \end{equation}
Therefore, 
    \begin{equation}
        \Phi_0 =  X^T(D_\Gamma - \Gamma)X.
    \end{equation}

\textbf{$\mathbf{\Phi_0}$ for the RBF Relative Kernel: }
With a RBF Relative Kernel, it is easier to start with the Lagrangian once we have approximated relative Kernel with the 2nd order Taylor expansion as
    \begin{equation}
        \mathcal{L} \approx
            -\sum_{i,j} \Gamma_{i,j} 
            \left[ 1 + \Tr(W^T(-\frac{1}{\sigma_i \sigma_j}A_{i,j})W) \right]
            - \Tr \left[ \Lambda(W^TW - I) \right].
    \end{equation}
The gradient of the Lagrangian is therefore
    \begin{equation}
        \nabla_W \mathcal{L} \approx
            \left[ \sum_{i,j} \Gamma_{i,j} (\frac{2}{\sigma_i \sigma_j}A_{i,j}) \right] W
            - 2 W \Lambda.
    \end{equation}    
Setting the gradient to 0, we get 
     \begin{equation}
            \left[ \sum_{i,j} (\frac{1}{\sigma_i \sigma_j} \Gamma_{i,j} A_{i,j}) \right] W
            = W \Lambda.
    \end{equation}       
If we let $\Sigma_{i,j} = \frac{1}{\sigma_i \sigma_j}$ and $\Psi = \Sigma \odot \Gamma$, then we end up with     
     \begin{equation}
        4 \left[ X^T (D_{\Psi} - \Psi) X \right] W
            = W \Lambda.
    \end{equation}       
\textbf{$\mathbf{\Phi_0}$ for the Squared Kernel: }
With a Squared Kernel, $(\mathbf{a},\mathbf{b})$ uses $x_i-x_j$, therefore Eq.~(\ref{eq:phi_0_form_1}) is use. Since $f(\beta) = \beta$, the gradient with respect to $\beta$ is 
    \begin{equation}
        -sign(4 \nabla_\beta f(\beta)) = -sign(4) = -1.
    \end{equation}
Therefore, 
    \begin{equation}
        \Phi_0 = -X^T(D_\Gamma - \Gamma)X.
    \end{equation}

\textbf{$\mathbf{\Phi_0}$ for the Multiquadratic Kernel: }
With a Multiquadratic Kernel, $(\mathbf{a},\mathbf{b})$ uses $x_i-x_j$, therefore Eq.~(\ref{eq:phi_0_form_1}) is use. Since $f(\beta) = \sqrt{\beta + c^2}$, the gradient with respect to $\beta$ is 
    \begin{equation}
        -sign(4 \nabla_\beta f(\beta)) = -sign(\frac{4}{2}(\beta + c^2)^{-1/2}) = -1.
    \end{equation}
Therefore, 
    \begin{equation}
        \Phi_0 =  -X^T(D_\Gamma - \Gamma)X.
    \end{equation}

\end{appendices}

\begin{appendices}
\section{Derivation for each $\Phi$}
\label{app:deriv_phi}
Using Eq.~(\ref{eq:eig_decomp_1}), we know that
    \begin{equation}
        \Phi = -\frac{1}{2}\sum_{i,j} \Gamma_{i,j}[\nabla_\beta f(\beta)] A_{i,j}. 
    \end{equation}
If we let $\Psi=-\Gamma_{i,j} [\nabla_\beta f(\beta)]$ then $\Phi$ can also be written as
    \begin{equation}
        \Phi = \frac{1}{2}\sum_{i,j} \Psi_{i,j} A_{i,j}. 
    \end{equation}
If $\mathbf{a}$ and $\mathbf{b}$ are both defined as $x_i-x_j$, then 
    \begin{equation}
        \Phi = 2 X^T(D_\Psi - \Psi) X.
        \label{eq:phi_form_1}
    \end{equation}
However, if $\mathbf{a}$ and $\mathbf{b}$ are defined as $(x_i,x_j)$, then 
    \begin{equation}
        \Phi = X^T\Psi X.
        \label{eq:phi_form_2}
    \end{equation}
Therefore, to compute $\Phi$, the key is to first determine the ($\mathbf{a}$ , $\mathbf{b}$) based on the kernel and then find the appropriate $\Psi$.

\textbf{$\mathbf{\Phi}$ for the Linear Kernel: }
With a Linear Kernel, $(\mathbf{a},\mathbf{b})$ uses $(x_i,x_j)$, therefore Eq.~(\ref{eq:phi_form_2}) is use. Since $f(\beta) = \beta$, the gradient with respect to $\beta$ is 
    \begin{equation}
        \Phi = -\frac{1}{2}\sum_{i,j} \Gamma_{i,j}[\nabla_\beta f(\beta)] A_{i,j} 
        = -\frac{1}{2}\sum_{i,j} \Gamma_{i,j} A_{i,j}.
    \end{equation}
Since, we are only interested in the eigenvectors of $\Phi$ only the sign of the constants are necessary. Therefore, 
    \begin{equation}
        \Phi = sign(-1) X^T\Gamma X = -X^T\Gamma X.
    \end{equation}

\textbf{$\mathbf{\Phi}$ for the Polynomial Kernel: }
With a Polynomial Kernel, $(\mathbf{a},\mathbf{b})$ uses $(x_i,x_j)$, therefore Eq.~(\ref{eq:phi_form_2}) is use. Since $f(\beta) = (\beta + c)^p$, the gradient with respect to $\beta$ is 
    \begin{equation}
         \Phi = -\frac{1}{2}\sum_{i,j} \Gamma_{i,j}[\nabla_\beta f(\beta)] A_{i,j} 
         = -\frac{1}{2}\sum_{i,j} \Gamma_{i,j}[p(\beta+c)^{p-1}] A_{i,j}.
    \end{equation}
Since $p$ is a constant, and $K_{XW,p-1} = (\beta+c)^{p-1}$ is the polynomial kernel itself with power of $(p-1)$, $\Psi$ becomes
    \begin{equation}
        \Psi = \Gamma \odot K_{XW,p-1},
    \end{equation}
and
    \begin{equation}
        \Phi = sign(-p) X^T \Psi X = -X^T \Psi X
    \end{equation}
\textbf{$\mathbf{\Phi}$ for the Gaussian Kernel: }
With a Gaussian Kernel, $(\mathbf{a},\mathbf{b})$ uses $x_i-x_j$, therefore Eq.~(\ref{eq:phi_0_form_1}) is use. Since $f(\beta) = e^{-\frac{\beta}{2\sigma^2}}$, the gradient with respect to $\beta$ is 
    \begin{equation}
        \Phi = -\frac{1}{2} \sum_{i,j} \Gamma_{i,j}[\nabla_\beta f(\beta)] A_{i,j} 
         = -\frac{1}{2} \sum_{i,j} \Gamma_{i,j}[-\frac{1}{2\sigma^2}e^{-\frac{\beta}{2\sigma^2}}] A_{i,j}= \frac{1}{4\sigma^2} \sum_{i,j} \Gamma_{i,j}[K_{XW}]_{i,j} A_{i,j}. 
    \end{equation}
If we let $\Psi=\Gamma \odot K_{XW}$, then
    \begin{equation}
        \Phi =  sign(\frac{2}{4\sigma^2}) X^T(D_\Psi - \Psi)X = X^T(D_\Psi - \Psi)X.
    \end{equation}
    
\textbf{$\mathbf{\Phi}$ for the Squared Kernel: }
With a Squared Kernel, $(\mathbf{a},\mathbf{b})$ uses $x_i-x_j$, therefore Eq.~(\ref{eq:phi_0_form_1}) is use. Since $f(\beta) = \beta$, the gradient with respect to $\beta$ is 
    \begin{equation}
         \Phi = -\frac{1}{2} \sum_{i,j} \Gamma_{i,j}[\nabla_\beta f(\beta)] A_{i,j} 
         = -\frac{1}{2} \sum_{i,j} \Gamma_{i,j}A_{i,j}.    
    \end{equation}
Therefore, 
    \begin{equation}
        \Phi = sign(-1) X^T(D_\Gamma - \Gamma)X = - X^T(D_\Gamma - \Gamma)X.
    \end{equation}   
    
\textbf{$\mathbf{\Phi}$ for the Multiquadratic Kernel: }
With a Multiquadratic Kernel, $(\mathbf{a},\mathbf{b})$ uses $x_i-x_j$, therefore Eq.~(\ref{eq:phi_0_form_1}) is use. Since $f(\beta) = \sqrt{\beta+c^2}$, the gradient with respect to $\beta$ is 
    \begin{equation}
        \Phi = -\frac{1}{2} \sum_{i,j} \Gamma_{i,j}[\nabla_\beta f(\beta)] A_{i,j} 
         = -\frac{1}{2} \sum_{i,j} \Gamma_{i,j}[\frac{1}{2}(\beta+c^2)^{-1/2}] A_{i,j}= -\frac{1}{4} \sum_{i,j} \Gamma_{i,j}[K_{XW}]_{i,j}^{(-1)} A_{i,j}. 
    \end{equation}
If we let $\Psi=\Gamma \odot K_{XW}^{(-1)}$, then
    \begin{equation}
        \Phi =  sign(-\frac{1}{4}) X^T(D_\Psi - \Psi)X = -X^T(D_\Psi - \Psi)X.
    \end{equation}   
    
\textbf{$\mathbf{\Phi_0}$ for the RBF Relative Kernel: }
With a RBF Relative Kernel, we start with the initial Lagrangian 
    \begin{equation}
        \mathcal{L} = -\sum_{i,j} \Gamma_{i,j} e^{-\frac{Tr(W^TA_{i,j}W)}{2\sigma_i \sigma_j}} -
        \Tr(\Lambda(W^TW - I))
    \end{equation}
where the gradient becomes
    \begin{equation}
        \nabla_W \mathcal{L} = \sum_{i,j} 
        \frac{1}{\sigma_i \sigma_j} \Gamma_{i,j} e^{-\frac{Tr(W^TA_{i,j}W)}{2\sigma_i \sigma_j}}
        A_{i,j} W - 2 W \Lambda.
    \end{equation}
If we let $\Sigma_{i,j} = \frac{1}{\sigma_i \sigma_j}$ then we get 
     \begin{equation}
        \nabla_W \mathcal{L} = \sum_{i,j} 
        \Psi_{i,j} 
        A_{i,j} W - 2 W \Lambda,
    \end{equation}   
where $\Psi_{i,j}=\Sigma_{i,j} \Gamma_{i,j} K_{XW_{i,j}}$. If we apply Appendix \ref{app:xi-xj} and set the gradient to 0, then we get
     \begin{equation}
        4 \left[ X^T (D_{\Psi} - \Psi) X \right] 
        W = 2 W \Lambda.
    \end{equation}   
From here, we see that it has the same form as the Gaussian kernel, with $\Psi$ defined as $\Psi = \Sigma \odot \Gamma \odot K_{XW}$.
\end{appendices}

\begin{appendices}
\section{Computing the Hessian for the Taylor Series}
First we compute the gradient and the Hessian for $\beta(W)$ where
    \begin{align}
    \beta(W) &= a^TWW^Tb,\\
    \beta(W) &= \Tr(W^Tba^TW),\\
    \nabla_W \beta(W) &= [ba^T+ab^T]W,\\
    \nabla_{W,W} \beta(W) &= [ba^T+ab^T],\\
    \nabla_{W,W} \beta(W=0) &= [ba^T+ab^T].\\
    \end{align}
Next, we compute the gradient and Hessian for $f(\beta(W))$ where
    \begin{align}
    f(\beta(W)) &= f(a^TWW^Tb),\\
    f(\beta(W)) &= f(\Tr(W^Tba^TW)),\\
    f'(\beta(W)) &= \nabla_{\beta} f(\beta(W))[ba^T+ab^T]W = 
        \nabla_{\beta} f(\beta(W)) \nabla_W \beta(W)\\
    f''(\beta(W) &= \nabla_{\beta,\beta} f(\beta(W))[ba^T+ab^T]W(...) + \nabla_{\beta} f(\beta(W))[ba^T+ab^T]\\
    f''(\beta(W=0)) &= 0 + \nabla_{\beta} f(\beta(W))\nabla_{W,W} \beta(W=0)\\
    f''(\beta(W=0)) &= \nabla_{\beta} f(\beta(W))\nabla_{W,W} \beta(W=0)\\
    f''(0) &= \mu A_{i,j}.
    \end{align}
Using Taylor Series the gradient of the Lagrangian is approximately
    \begin{align}
    \nabla_{W} \mathcal{L} &\approx -\sum_{i,j} \Gamma_{i,j} f''(0) W - 2W\Lambda,\\
    \nabla_{W} \mathcal{L} &\approx -\mu \sum_{i,j} \Gamma_{i,j} A_{i,j} W - 2W\Lambda.
    \end{align}
Setting the gradient of the Lagrangian to 0 and combining the constant 2 to $\mu$, it yields the relationship
    \begin{align}
    \left[ -\mu \sum_{i,j} \Gamma_{i,j} A_{i,j} \right] W = W\Lambda,\\
    \Phi_0 W = W\Lambda.\\
    \end{align}
\end{appendices}

\begin{appendices}
\section{Lemma 2 Proof}
\label{eq:lemma2condition}
To proof Lemma 2, we must relate the concept of eigengap to the conditions of
\begin{equation}
  \tmop{Tr} ( Z^T \nabla^2_{W W} \mathcal{L} ( W^{\ast}, \Lambda^{\ast}) Z)
  \geq 0 \nocomma \begin{array}{lllll}
    & \forall & Z \neq 0 & \tmop{with} & \nabla h ( W^{\ast})^T Z = 0
    \label{eq:orig_inequality}
  \end{array} . 
\end{equation}
Given the constraint $h ( W) = W^T W - I$, we start by computing the constrain
$\nabla h ( W^{\ast})^T Z = 0$. Given
\begin{equation}
  \nabla h ( W^{\ast})^T Z = \begin{array}{l}
    \lim\\
    t \rightarrow 0
  \end{array} \frac{\partial}{\partial t} h ( W \noplus + t Z),
\end{equation}
the constraint becomes
\begin{equation}
  \begin{array}{lll}
    \nabla h ( W^{\ast})^T Z = 0 & = & \begin{array}{l}
      \lim\\
      t \rightarrow 0
    \end{array} \frac{\partial}{\partial t} [ ( W \noplus + t Z)^T ( W \noplus
    + t Z) - I],\\
    0 & = & \begin{array}{l}
      \lim\\
      t \rightarrow 0
    \end{array} \frac{\partial}{\partial t} [ ( W^T \noplus W + t W^T Z + t
    Z^T W + t^2 Z^T Z) - I],\\
    0 & = & \begin{array}{l}
      \lim\\
      t \rightarrow 0
    \end{array} W^T Z + Z^T W + 2 t Z^T Z.
  \end{array}
\end{equation}
By setting the limit to 0, an important relationship emerges as
\begin{equation}
  \begin{array}{lll}
    0 & = & W^T Z + Z^T W.
    \label{eq:constraint_2nd}
  \end{array}
\end{equation}
Given a full rank operator $\Phi$, its eigenvectors must span the complete
$\mathcal{R}^d$ space. If we let $W$ and $\bar{W}$ represent the eigenvectors
chosen and not chosen respectively from Algorithm 1, and let $B$ and $\bar{B}$
be scambling matrices, then the matrix $Z \in \mathcal{R}^{d \times q}$ can be
rewritten as
\begin{equation}
  Z = W B + \bar{W}  \bar{B} .
  \label{eq:z_equal_to}
\end{equation}
It should be noted that since $W$ and $\bar{W}$ are eigenvalues of the
symmetric matrix $\Phi$, they are orthogonal to each other, i.e., $W^T \bar{W}
= 0$. Furthermore, if we replace $Z$ in Eq.~(\ref{eq:constraint_2nd}) with Eq.~(\ref{eq:z_equal_to}), we get the
condition
\begin{equation}
  \begin{array}{lll}
    0 & = & W^T ( W B + \bar{W}  \bar{B}) + ( W B + \bar{W}  \bar{B})^T W\\
    0 & = & B + B^T .
    \label{eq:antisym}
  \end{array}
\end{equation}
From Eq.~(\ref{eq:antisym}), we observe that $B$ must be a antisymmetric matrix because $B = -B^T$. Next, we work
to compute the inequality of of Eq.~(\ref{eq:orig_inequality}) by noting that
\begin{equation}
  \nabla^2_{W W} \mathcal{L} ( W, \Lambda) Z = \begin{array}{l}
    \lim\\
    t \rightarrow 0
  \end{array} \frac{\partial}{\partial t} \nabla \mathcal{L} ( W + t Z)
  \nosymbol .
\end{equation}
Also note that Lemma 1 has already computed $\nabla_W \mathcal{L} ( W)$ as
\begin{equation}
  \nabla_W \mathcal{L} ( W) = -\frac{1}{2}\left[ \sum_{i, j} \Gamma_{i, j} f' ( \beta)
  A_{i, j} \right] W - W \Lambda .
\end{equation}
Since we need $\nabla_W \mathcal{L}$ to be a function of $W + t Z$ with $t$ as
the variable, we change $\beta ( W)$ into $\beta ( W + t Z)$ with
\begin{equation}
  \begin{array}{lll}
    \beta ( W + t Z) & = & \tmmathbf{a} ( W + t Z) ( W + t Z)^T
    \tmmathbf{b},\\
    & = & \tmmathbf{a}^T W W^T \tmmathbf{b}+ [ \tmmathbf{a}^T ( W Z^T + Z
    W^T) \tmmathbf{b}] t + [ \tmmathbf{a}^T Z Z^T \tmmathbf{b}] t^2,\\
    & = & \beta + c_1 t + c_2 t^2,
    \label{eq:beta}
  \end{array}
\end{equation}
where $\beta$, $c_1$, and $c_2$ are constants with respect to $t$. Using the
$\beta$ from Eq.~(\ref{eq:beta}) with $\nabla_W \mathcal{L}$, we get
\begin{equation}
  \nabla^2_{W W} \mathcal{L} ( W, \Lambda) Z = \begin{array}{l}
    \lim\\
    t \rightarrow 0
  \end{array} \frac{\partial}{\partial t} \left[-\frac{1}{2} \sum_{i, j}  \Gamma_{i, j}
  f' ( \beta + c_1 t + c_2 t^2) A_{i, j} \right] ( W + t Z) - ( W + t Z)
  \Lambda .
\end{equation}
If we take the derivative with respect to $t$ and then set the limit to 0, we
get
\begin{equation}
  \nabla^2_{W W} \mathcal{L} ( W, \Lambda) Z = \left[-\frac{1}{2} \sum_{i, j} \Gamma_{i,
  j} f'' ( \beta) c_1 A_{i, j} \right] W + \left[ -\frac{1}{2} \sum_{i, j} \Gamma_{i, j}
  f' ( \beta) A_{i, j} \right] Z - Z \Lambda .
\end{equation}
Next, we notice the definition of $\Phi = -\frac{1}{2} \sum \Gamma_{i, j} f' ( \beta)
A_{i, j}$ \ from Lemma 1, the term $\tmop{Tr} ( Z^T \nabla^2_{W W} \mathcal{L}
( W, \Lambda) Z)$ can now be expressed as 3 separate terms as
\begin{equation}
\tmop{Tr} ( Z^T \nabla^2_{W W} \mathcal{L} ( W, \Lambda) Z) =\mathcal{T}_1
   +\mathcal{T}_2 +\mathcal{T}_3,
   \label{eq:3_terms}
\end{equation}
where
    \begin{align}
        \mathcal{T}_1 & = \tmop{Tr} \left( Z^T \left[ -\frac{1}{2}\sum_{i, j} \Gamma_{i,
     j} f'' ( \beta) c_1 A_{i, j} \right] W \right),\\
          \mathcal{T}_2 & = \tmop{Tr} ( Z^T \Phi Z),
          \label{eq:app:T2}
          \\
        \mathcal{T}_3 & = - \tmop{Tr} ( Z^T Z \Lambda).
        \label{eq:app:T3}
    \end{align}
Since $\mathcal{T}_1$ cannot be further simplified, the concentration will be
on $\mathcal{T}_2$ and $\mathcal{T}_3$. If we let $\bar{\Lambda}$ and
$\Lambda$ be the corresponding eigenvlaue matrices associated with $\bar{W}$
and $W$, by replacing $Z$ in $\mathcal{T}_2$ from Eq.~(\ref{eq:app:T2}), we get
\[ \begin{array}{lll}
     \tmop{Tr} ( Z^T \Phi Z) & = & \tmop{Tr} ( ( W B + \bar{W} \bar{B})^T \Phi
     ( W B + \bar{W} \bar{B}))\\
     & = & \tmop{Tr} ( \nobracket B^T W^T \Phi W B + \bar{B}^T \bar{W}^T \Phi
     W B + B^T W^T \Phi \bar{W} \bar{B} + \bar{B}^T \bar{W}^T \Phi \bar{W}
     \bar{B}) \nobracket\\
     & = & \tmop{Tr} ( \nobracket B^T W^T W \Lambda B + \bar{B}^T \bar{W}^T W
     \Lambda B + B^T W^T \bar{W} \bar{\Lambda} \bar{B} + \bar{B}^T \bar{W}^T
     \bar{W} \bar{\Lambda} \bar{B}) \nobracket\\
     & = & \tmop{Tr} ( \nobracket B^T \Lambda B + 0 + 0 + \bar{B}^T
     \bar{\Lambda} \bar{B}) \nobracket\\
     & = & \tmop{Tr} ( \nobracket B^T \Lambda B + \bar{B}^T \bar{\Lambda}
     \bar{B}) \nobracket .
   \end{array} \]
By replacing $Z$ from $\mathcal{T}_3$ from Eq.~(\ref{eq:app:T3}), we get
\[ \begin{array}{lll}
     - \tmop{Tr} ( Z^T Z \Lambda) & = & - \tmop{Tr} ( ( W B + \bar{W}
     \bar{B})^T ( W B + \bar{W} \bar{B}) \Lambda)\\
     & = & - \tmop{Tr} ( B^T W^T W B \Lambda + \bar{B}^T \bar{W}^T W B
     \Lambda + B^T W^T \bar{W} \bar{B} \Lambda + \bar{B}^T \bar{W}^T \bar{W}
     \bar{B} \Lambda)\\
     & = & - \tmop{Tr} ( B^T B \Lambda + 0 + 0 + \bar{B}^T \bar{B} \Lambda)\\
     & = & - \tmop{Tr} ( B \Lambda B^T + \bar{B} \Lambda \bar{B}^T) .
   \end{array} \]
The inequality that satisfies the 2nd order condition can now be written as 
\begin{equation}
  \tmop{Tr} ( B^T \Lambda B) + \tmop{Tr} (  \bar{B}^T \bar{\Lambda} \bar{B}) -
  \tmop{Tr} ( B \Lambda B^T) - \tmop{Tr} ( \bar{B} \Lambda \bar{B}^T)
  +\mathcal{T}_1 \geq 0.
  \label{eq:inequality_2}
\end{equation}
Since $B$ is an antisymmetric matrix, $B^T = - B$, and therefore $\tmop{Tr} (
B \Lambda B^T) = \tmop{Tr} ( B^T \Lambda B)$. From this Eq.~(\ref{eq:inequality_2}) can be
rewritten as
\begin{equation}
  \tmop{Tr} ( B^T \Lambda B) - \tmop{Tr} ( B^T \Lambda B) + \tmop{Tr} ( 
  \bar{B}^T \bar{\Lambda} \bar{B}) - \tmop{Tr} ( \bar{B} \Lambda \bar{B}^T)
  +\mathcal{T}_1 \geq 0.
\end{equation}
With the first two terms canceling each other out, the inequality can be
rewritten as
\begin{equation}
  \tmop{Tr} (  \bar{B}^T \bar{\Lambda} \bar{B}) - \tmop{Tr} ( \bar{B} \Lambda
  \bar{B}^T) \geq \mathcal{T}_1 .
\end{equation}
With this inequality, the terms can be further bounded by
\[ \tmop{Tr} ( \bar{B}^T \bar{\Lambda} \bar{B}) \geq \begin{array}{l}
     \min\\
     i
   \end{array} \bar{\Lambda}_i \tmop{Tr} ( \bar{B} \bar{B}^T) \]
\[ \tmop{Tr} ( \bar{B} \Lambda \bar{B}^T) \geq \begin{array}{l}
     \max\\
     j
   \end{array} \Lambda_j \tmop{Tr} ( \bar{B}^T \bar{B}) \]
Noting that since $\tmop{Tr} ( \bar{B} \bar{B}^T) = \tmop{Tr} ( \bar{B}^T
\bar{B})$, we treat it as a constant value of $\alpha$. With this the
inequality can be rewritten as
\[ \left( \begin{array}{l}
     \min\\
     i
   \end{array} \bar{\Lambda}_i - \begin{array}{l}
     \max\\
     j
   \end{array} \Lambda_j \right) \geq \frac{1}{\alpha} \mathcal{T}_1 . \]

\end{appendices}

\end{document}